\documentclass{article}

\usepackage[preprint]{neurips_2020}

\usepackage{neurips_2020}

\usepackage[utf8]{inputenc} 
\usepackage[T1]{fontenc}    
\usepackage{hyperref}       
\usepackage{url}            
\usepackage{booktabs}       
\usepackage{amsfonts}       
\usepackage{nicefrac}       
\usepackage{microtype}      

\usepackage{subfigure}
\usepackage{booktabs} 
\usepackage{graphicx}
\usepackage{boldline}
\usepackage{amsthm}
\usepackage[cmex10]{amsmath}
\usepackage{amssymb}
\usepackage{multirow}
\usepackage{natbib}
\usepackage{boldline}
\usepackage{siunitx}
\usepackage{bbm}

\usepackage{algorithm}
\usepackage{algorithmic}

\usepackage{makecell}

\newtheorem{proposition}{Proposition}

\newtheorem{corollary}{Corollary}

\usepackage{amsmath,amsfonts,bm}

\def\eqref#1{equation~\ref{#1}}

\def\1{\bm{1}}

\DeclareMathAlphabet{\mathsfit}{\encodingdefault}{\sfdefault}{m}{sl}
\SetMathAlphabet{\mathsfit}{bold}{\encodingdefault}{\sfdefault}{bx}{n}

\newcommand{\KL}{D_{\mathrm{KL}}}

\DeclareMathOperator*{\argmax}{arg\,max}
\DeclareMathOperator*{\argmin}{arg\,min}

\title{Learning Generative Models using Denoising Density Estimators}

\author{%
  Siavash A. Bigdeli \\
  CSEM, Neuchâtel\\ Switzerland\\
  \\
	\And
	Geng Lin \\
  University of Maryland\\ College Park, USA \\
		\AND
	Tiziano Portenier \\
  ETHZ, Zurich\\ Switzerland \\
	  \And
	L. Andrea Dunbar \\
  CSEM, Neuchâtel\\ Switzerland \\
		\And
	Matthias Zwicker \\
  University of Maryland\\ College Park, USA \\
}

\begin{document}

\maketitle

\begin{abstract}

Learning probabilistic models that can estimate the density of a given set of samples, and generate samples from that density, is one of the fundamental challenges in unsupervised machine learning. We introduce a new generative model based on denoising density estimators (DDEs), which are scalar functions parameterized by neural networks, that are efficiently trained to represent kernel density estimators of the data. Leveraging DDEs, our main contribution is a novel technique to obtain generative models by minimizing the KL-divergence directly. We prove that our algorithm for obtaining generative models is guaranteed to converge to the correct solution. Our approach does not require specific network architecture as in normalizing flows, nor use ordinary differential equation solvers as in continuous normalizing flows. Experimental results demonstrate substantial improvement in density estimation and competitive performance in generative model training.

\end{abstract}

\section{Introduction}
Learning generative probabilistic models from raw data is one of the fundamental problems in unsupervised machine learning. 
These models enable sampling from the probability density represented by the input data, or also performing density estimation and inference of latent variables. 
Recently, the use of deep neural networks has led to significant advances in this area. For example, generative adversarial networks~\citep{Goodfellow2014GAN} can be trained to sample very high dimensional densities, but they do not provide density estimation or inference. Inference in Boltzman machines~\citep{salakhutdinov2009deep} is tractable only under approximations~\citep{WELLING200319}. Variational autoencoders~\citep{Kingma2014VAE} provide functionality for both (approximate) inference and sampling. Finally, normalizing flows~\citep{Dinh2014NICE} perform all three operations (sampling, density estimation, inference) efficiently.

In this paper we introduce a novel type of generative model based on what we call denoising density estimators (DDEs), which supports efficient sampling and density estimation. Our approach to construct a sampler is straightforward: assuming we have a density estimator that can be efficiently trained and evaluated, we learn a sampler by forcing its generated density to be the same as the input data density via minimizing their Kullback-Leibler (KL) divergence. In particular, we use the {\em reverse} KL divergence, which avoids saddle points when the two distributions are non-overlapping. In our approach, the density estimator is derived from the theory of denoising autoencoders, hence our term {\em denoising density estimator}. Compared to normalizing flows, a key advantage of our theory is that it does not require any specific network architecture, except differentiability, and we do not need to solve ordinary differential equations (ODE) like in continuous normalizing flows. In summary, our main contribution is a novel approach to obtain a generative model by explicitly estimating the energy (un-normalized density) of the generated and true data distributions and minimizing the statistical divergence of these densities. 

\section{Related Work}
Generative adversarial networks~\citep{Goodfellow2014GAN} are currently the most widely studied type of generative probabilistic models for very high dimensional data.
GANs are often difficult to train, however, and they can suffer from mode-collapse, sparking renewed interest in alternatives. A common approach is to formulate these models as mappings between a latent space and the data domain, and one way to categorize them is to consider the constraints on this mapping. In normalizing flows~\citep{Dinh2014NICE,Rezend2015NFl} the mapping is invertible and differentiable, such that the data density can be estimated using the determinant of its Jacobian, and inference can be perfomed via the inverse mapping. Normalizing flows can be trained simply using maximum likelihood estimation (MLE)~\citep{Dinh2017NVP}. The challenge, however, is to design efficient computational structures for the required operations~\citep{Huang2018NAF,Kingma2018GLOW}. \citet{Chen2018NODE} and \citet{grathwohl2019ffjord} derive continuous normalizing flows by parameterizing the dynamics (the time derivative) of an ODE using a neural network,
but it comes at the cost of solving ODEs to produce outputs.
In contrast, in variational techniques the relation between the latent variables and data is probabilistic, usually expressed as a Gaussian likelihood function. Hence computing the marginal likelihood requires integration over latent space. To make this tractable, it is common to bound the marginal likelihood using the evidence lower bound~\citep{Kingma2014VAE}. 
Recently,~\citet{Li2018IMLE} proposed an approximate form of MLE, which they call implicit MLE (IMLE), that can also be performed without requiring invertible mappings. As a disadvantage, IMLE requires nearest neighbor queries in (high dimensional) data space.

Not all generative models include a latent space, for example autoregressive models~\citep{Oord2016PCNN} or denoising autoencoders (DAEs)~\citep{JMLR:v15:alain14a}. In particular,~\citet{JMLR:v15:alain14a} and~\citet{Saremi2019NeuralEB} use the well known relation between DAEs and the score of the corresponding data distributions~\citep{Vincent:2011:CSM,Raphan:2011:LSE} to construct an approximate Markov Chain sampling procedure.
Similarly, \citet{bigdeli2017image} and~\citet{bigdeli2017deep} use DAEs to learn the gradient of image densities for optimizing maximum a-posteriori problems in image restoration.
We build on DAEs, but formulate an estimator for the un-normalized, scalar density, rather than for the score (a vector field). This is crucial to allow us to train a generator instead of requiring Markov chain sampling, which has the disadvantages of requiring sequential sampling and producing correlated samples.

Instead of using a denoising objective, score-matching can also be achieved by minimizing Stein's loss for the true and estimated density gradients.
\citet{kingma2010regularized} use a regularized version of the loss to parametrize a product-of-experts model for images, and
\citet{li2019learning} train deep density estimators based on exponential family kernels.
These techniques require computation of third order derivatives, however, limiting the dimensionality of their models.
\citet{Song2019GMG} extend this approach by introducing a sliced score-matching objective that leads to more efficient training.
Unlike these techniques, DDEs are optimized using a denoising objective, hence they can be optimized without approximations, nor higher order derivatives. This allows us to efficiently train an exact generator that scales well with the data dimensionality.
In addition, \citet{Song2019GMG} formulate a generative model using Langevin dynamics,
which requires an iterative sampling procedure that provides exact sampling only asymptotically.
Similarly, \citet{dai2019exponential} use adversarial training to learn dynamics for generating samples.
Unlike these approaches, we do not require an iterative sampling scheme and our generator produces samples in single forward passes.

Other energy-based techniques for generative models include the work by~\citet{kim2016deep}, who use directed graphs to learn densities in latent space and to train their generator.
The approximation in this approach limits their generalization to complex and higher dimension datasets.
Using kernel exponential families, \citet{dai2019kernel} train a density estimator at the same time as their dual generator.
Similar to other score-matching optimizations, their approach requires quadratic computations with respect to the input dimensions at each gradient calculation.
Moreover, they only report generated results on 2D toy examples.
Table~\ref{tbl:comparison} summarizes the differences of our approach to GANs, Score-Matching, and Normalizing Flows.


\begin{table}[t]
\bgroup
\setlength{\tabcolsep}{3.5pt}
\begin{center}
\begin{tabular}[c]{l c c c c }
\hlineB{3}
Property & GAN & Score-Matching & Normalizing Flows & Ours \\
\hline
Provides density & (-) & - & \checkmark & \checkmark \\
Forward sampling model & \checkmark & iterative & \checkmark & \checkmark \\
Exact sampling & \checkmark & asymptotic & \checkmark & \checkmark \\
Free net architecture & \checkmark & \checkmark & - & \checkmark \\
\hlineB{3}
\end{tabular}
\end{center}
\egroup
\caption
{
Comparison of different deep generative approaches based on GANs, Score-Matching, Normalizing Flows, and our proposed technique. Adversarial density estimation can be achieved using the approach by~\citet{abbasnejad2019generative} using a suitable training objective. 
}
\label{tbl:comparison}
\end{table}

\section{Denoising Density Estimators (DDEs)}
\label{sec:dde}

First we describe how to estimate a density using a variant of denoising autoencoders (DAEs). More precisely, this approach allows us to obtain the density smoothed by a Gaussian kernel, which is equivalent to kernel density estimation~\citep{parzen1962}, up to a normalizing factor. Originally, the optimal DAE $r: \mathbb{R}^n \rightarrow \mathbb{R}^n$~\citep{Vincent:2011:CSM,JMLR:v15:alain14a} is defined as the function minimizing the following denoising loss, 
\begin{align}
&\mathcal{L}_{\mathrm{DAE}}(r;p, \sigma_{\eta}) = \mathbb{E}_{x\sim p,\eta \sim \mathcal{N}(0,\sigma_\eta^2)}\left[ \| r(x+\eta) - x\|^2 \right],
\label{eq:denoisingloss}
\end{align}
where the data $x$ is distributed according to a density $p$ over $\mathbb{R}^n$, and $\eta \sim \mathcal{N}(0,\sigma_\eta^2)$ represents $n$-dimensional, isotropic additive Gaussian noise with variance $\sigma_\eta^2$. It has been shown~\citep{robbins1956,Raphan:2011:LSE,bigdeli2017image} that the optimal DAE $r^*(x)$ minimizing $\mathcal{L}_{\mathrm{DAE}}$ can be expressed as follows, which is also known as Tweedie's formula,
\begin{align}
r^*(x) = x + \sigma_\eta^2 \nabla_x \log \tilde{p}(x),
\end{align}
where $\nabla_x$ is the gradient with respect to the input $x$, $\tilde{p}(s) = [p*k](x)$ denotes the convolution between the data and noise distributions $p(x)$, and $k=\mathcal{N}(0,\sigma_\eta^2)$. 
Inspired by this result, we reformulate the DAE-loss as a noise estimation loss,
\begin{align}
&\mathcal{L}_{\mathrm{NEs}}(f;p, \sigma_{\eta}) = \mathbb{E}_{x\sim p,\eta \sim \mathcal{N}(0,\sigma_\eta^2)}\left[ \|  f(x+\eta) +\eta / \sigma_\eta^2\|^2 \right],
\label{eq:noiseestimationloss}
\end{align}
where $f: \mathbb{R}^n \rightarrow \mathbb{R}^n$ is a vector field that estimates the noise vector $-\eta / \sigma_\eta^2$.
Similar to \citet{Vincent:2011:CSM} and \citet{JMLR:v15:alain14a}, we formulate the following proposition and provide the proof in the supplementary material:

\begin{proposition}
There is a unique minimizer $f^*(x) = \argmin_f \mathcal{L}_{\mathrm{NEs}}(f;p,\sigma_{\eta})$ that satisfies
\begin{align}
f^*(x) = \nabla_x \log \tilde{p}(x) = \nabla_x \log [p*k](x).
\end{align}
That is, the optimal estimator corresponds to the gradient of the logarithm of the Gaussian smoothed density $\tilde{p}(x)$, that is, the score of the density.
\label{prop:nes}
\end{proposition}

A key observation is that the desired vector-field $f^*$ is the gradient of a scalar function and conservative. Hence we can write the noise estimation loss in terms of a scalar function $s: \mathbb{R}^n \rightarrow \mathbb{R}$ instead of the vector field $f$, which we call the denoising density estimation loss, 
\begin{align}
&\mathcal{L}_{\mathrm{DDE}}(s;p, \sigma_{\eta}) = \mathbb{E}_{x\sim p,\eta \sim \mathcal{N}(0,\sigma_\eta^2)}\left[ \|  \nabla_x s(x+\eta) +\eta / \sigma_\eta^2\|^2 \right].
\label{eq:ddeloss}
\end{align}
A similar formulation has recently been proposed by~\citet{Saremi2019NeuralEB}. Our terminology is motivated by the following corollary:
\begin{corollary} The minimizer $s^*(x) = \argmin_s \mathcal{L}_{\mathrm{DDE}}(s;p)$ satisfies
\begin{align}
s^*(x) = \log \tilde{p}(x) + C, 
\label{eq:dde}
\end{align}
with some constant $C \in \mathbb{R}$.
\label{cor:dde}
\end{corollary}
\begin{proof}
From Proposition~\ref{prop:nes} and the definition of $\mathcal{L}_{\mathrm{DDE}}(s;p)$ we know that $\nabla_x s^*(x) = \nabla_x \log \tilde{p}(x)$, which leads immediately to the corollary.
\end{proof}
In summary, we have shown how modifying the denoising autoencoder loss (Eq.~\ref{eq:denoisingloss}) into a noise estimation loss based on the gradients of a scalar function (Eq.~\ref{eq:ddeloss}) allows us to derive a density estimator (Corollary~\ref{cor:dde}), which we call the denoising density estimator (DDE). In practice, we approximate the DDE using a neural network $s(x; \theta)$.
For illustration, Figure~\ref{fig:2Ddensity} shows 2D distribution examples, which we approximate using a DDE implemented as a multi-layer perceptron.

\section{Learning Generative Models using DDEs}
\label{sec:generativemodels} 

By leveraging DDEs, our key contribution is to formulate a novel training algorithm to obtain generators for given densities, which can be represented by a set of samples or as a continuous function. In either case, we denote the smoothed data density $\tilde{p}$, which is obtained by training a DDE in case the input is given as a set of samples as described in Section~\ref{sec:dde}. We express our samplers using mappings $x=g(z)$, where $x\in \mathrm{R}^n$, and $z \in \mathrm{R}^m$ (usually $n > m$) is a latent variable, which typically has a standard normal distribution. In contrast to normalizing flows, $g(z)$ does not need to be invertible. Let us denote the distribution of $x$ induced by the generator as $q$, that is $q \sim g(z)$, and also its Gaussian smoothed version $\tilde{q} = q * k$. 

We obtain the generator by minimizing the KL divergence $\KL(\tilde{q}||\tilde{p})$ between the density induced by the generator $\tilde{q}$ and the data density $\tilde{p}$. Our algorithm is based on the following observation:

\begin{proposition}
Given a scalar function $\Delta: \mathbb{R}^n \rightarrow \mathbb{R}$ that satisfies the following conditions:
\begin{align}
\KL(\tilde{q}||\tilde{p}) = \left< \tilde{q}, \log \tilde{q} -\log \tilde{p}\right>
&> \left< \tilde{q}+\Delta, \log \tilde{q} -\log \tilde{p}\right>, \label{eq:assumption1} \\
\left<\Delta, \mathbbm{1} \right> &= 0, \label{eq:assumption2} \\
\Delta^2 &< \epsilon, \quad \mbox{(pointwise exponentiation)} \label{eq:assumption3}
\end{align}
then $\KL(\tilde{q}||\tilde{p}) > \KL(\tilde{q}+\Delta ||\tilde{p})$ for small enough $\epsilon$.
\label{pro:klupdate}
\end{proposition}

\begin{proof}
We will use the first order approximation $\log (\tilde{q}+\Delta) = \log \tilde{q} + \Delta / \tilde{q} + o(\Delta^2)$, where the division is pointwise. Using $\left<\cdot,\cdot\right>$ to denote the inner product, we can write
\begin{align}
\KL(\tilde{q}+\Delta ||\tilde{p}) =& \left< \tilde{q}+\Delta, \log (\tilde{q}+\Delta) - \log \tilde{p}\right> \nonumber \\
=& \left< \tilde{q}+\Delta,  \log \tilde{q} + \Delta/\tilde{q} + o(\Delta^2) - \log \tilde{p}\right> \nonumber \\
=& \left< \tilde{q},  \log \tilde{q} - \log \tilde{p}\right> + 
 \left< \Delta,  \log \tilde{q} - \log \tilde{p}\right> \nonumber \\
& + \left< \tilde{q},  \Delta/\tilde{q} \right> +
\left< \Delta,  \Delta/\tilde{q} \right>  + O(\Delta^2).
\end{align}
This means
\begin{align}
\KL(\tilde{q}+\Delta ||\tilde{p}) - \KL(\tilde{q} ||\tilde{p}) = \left< \Delta,  \log \tilde{q} - \log \tilde{p}\right>
+ \left< \tilde{q},  \Delta/\tilde{q} \right> 
+ \left< \Delta,  \Delta/\tilde{q} \right>   + O(\Delta^2) < 0
\end{align}
because the first term on the right hand side is negative (first assumption in Equation~\ref{eq:assumption1}), the second term is zero (second assumption in Equation~\ref{eq:assumption2}), and the third and fourth terms are quadratic in $\Delta$ and can be ignored for $\Delta < \epsilon$ when $\epsilon$ is small enough (third assumption in Equation~\ref{eq:assumption3}).
\end{proof}
Based on the above observation, Algorithm~\ref{algo:generator} minimizes $\KL(\tilde{q}||\tilde{p})$ by iteratively computing updated densities $\tilde{q}+\Delta$ that satisfy the conditions from Proposition~\ref{pro:klupdate}, hence $\KL(\tilde{q}||\tilde{p}) > \KL(\tilde{q}+\Delta ||\tilde{p})$. This is guaranteed to converge to a global minimum, because $\KL(\tilde{q}||\tilde{p})$ is convex in $\tilde{q}$. 

At the beginning of each iteration in Algorithm~\ref{algo:generator} (Line 3), by definition $q$ is the density obtained by sampling our generator $x = g(z;\phi), z \sim \mathcal{N}(0,1)$ ($n$-dimensional standard normal distribution), and the generator is a neural network with parameters $\phi$. In addition, $\tilde{q} = q*k$ is defined as the density obtained by sampling $x = g(z;\phi) + \eta, z \sim \mathcal{N}(0,1), \eta \sim \mathcal{N}(0,\sigma_\eta^2)$. Finally, the DDE $s^{\tilde{q}}$ correctly estimates $\tilde{q}$, that is $\log \tilde{q}(x) = s^{\tilde{q}}(x) + C$. 
In each iteration, we update the generator such that its density is changed by a small $\Delta$ that satisfies the conditions from Proposition~\ref{pro:klupdate}. We achieve this by computing a gradient descent step of $\mathbb{E}_{x = g(z;\phi) + \eta} \left[ s^{\tilde{q}}(x) - \log \tilde{p}(x) \right] + C$ with respect to the generator parameters $\phi$ (Line 4). The constant $C$ can be ignored since we only need the gradient ($\tilde{q}$ always integrate to one after any generator update). A small enough learning rate guarantees that condition one (Equation~\ref{eq:assumption1}) in Proposition~\ref{pro:klupdate} is satisfied. The second condition (Equation~\ref{eq:assumption2}) is satisfied because we update the distribution by updating its generator, and the third condition (Equation~\ref{eq:assumption3}) is also satisfied under a small enough learning rate (and assuming the generator network is Lipschitz continuous). After updating the generator, we update the DDE to correctly estimate the new density produced by the updated generator (Line 6). Note that in practice, we perform fixed number of iterations (5-10 steps similar to GANs) to optimize the DDE, which did not lead to any instabilities.

Note that it is crucial in the first step in the iteration in Algorithm~\ref{algo:generator} that we sample using $g(z;\phi) + \eta$ and not $g(z;\phi)$. This allows us, in the second step, to use the updated $g(z;\phi)$ to train a DDE $s^{\tilde{q}}$ that exactly (up to a constant) matches the density generated by $g(z;\phi) + \eta$. Even though in this approach we only minimize the KL divergence with the ``noisy'' input density $\tilde{p}$, the sampler $g(z;\phi)$ still converges to a sampler of the underlying density $p$ in theory (exact sampling).

\begin{algorithm}
	\caption{Training steps for the generator. The input to the algorithm is a pre-trained optimal DDE on input data $\log \tilde{p}(x)$ and a learning rate $\delta$.}
	\label{algo:generator}
	\begin{algorithmic}[1]
		\STATE Initialize generator parameters $\phi$ 
		\STATE Initialize DDE $s^{\tilde{q}} = \argmin_{s} \mathcal{L}_{\mathrm{DDE}}(s; q, \sigma_{\eta}) $ with $q \sim g(z; \phi), z \sim \mathcal{N}(0,1)$
		\WHILE{not converged}
			\STATE $\phi = \phi - \delta \nabla_{\phi} \mathbb{E}_{x = g(z;\phi) + \eta} \left[ s^{\tilde{q}}(x) - \log \tilde{p}(x) \right]  $, with $z \sim \mathcal{N}(0,1), \eta \sim \mathcal{N}(0,\sigma_{\eta}^ 2)$
			\STATE // $q \sim g(z; \phi)$ now indicates the updated density using the updated $\phi$
			\STATE $s^{\tilde{q}} = \argmin_{s} \mathcal{L}_{\mathrm{DDE}}(s; q, \sigma_{\eta}) $ // In practice, we only take few optimization steps
			\STATE // $s^{\tilde{q}}$ is now the density (up to a constant) of $g(z;\phi) + \eta$
 		\ENDWHILE
	\end{algorithmic}
\end{algorithm}

\paragraph{Exact Sampling.}

Our objective involves reducing the KL divergence between the Gaussian smoothed generated density $\tilde{q}$ and the data density $\tilde{p}$. This also implies that the density $q$ obtained from sampling the generator $g(z;\phi)$ is identical with the data density $p$, without Gaussian smoothing, which can be expressed as the following corollary:

\begin{corollary}
Let $\tilde{p}$ and $\tilde{q}$ be related to densities $p$ and $q$, respectively, via convolutions using a Gaussian $k$, that is $\tilde{p} = p*k, \tilde{q} = q*k$. Then the smoothed densities $\tilde{p}$ and $\tilde{q}$ are the same if and only if the data density $p$ and the generated density $q$ are the same.
\end{corollary}

This follows immediately from the convolution theorem and the fact that the Fourier transform of Gaussian functions is non-zero everywhere, that is, Gaussian blur is invertible.

\paragraph{Relation to GANs.}

In the original GANs~\citep{Goodfellow2014GAN}, the generator is trained to minimize the Jensen-Shannon divergence between generated and real data distributions.
Our model is optimized to minimize the KL-divergence instead, which has been shown to achieve better likelihood scores compared to GANs~\citep{nowozin2016f}.
Moreover, we use the reverse KL-divergence loss in our training, which unlike forward KL-divergence, avoids saddle points when the two distributions are non-overlapping. This is because minimizing the reverse KL divergence can be reformulated as
\begin{align}
\argmin_{\tilde{q}} \KL(\tilde{q}||\tilde{p}) = \argmax_{\tilde{q}} E_{x\sim \tilde{q}}\left[ \log \tilde{p}(x)\right] + \mathcal{H}(\tilde{q}(x)),
\end{align}
which includes a term that attempts to maximize the entropy $\mathcal{H}$ of the generated distribution $\tilde{q}$.
Wasserstein-GANs address the same issue by using the Wasserstein distance between the two distributions to formulate their loss.
These models, however, require the discriminator network to guarantee Lipschitz continuity, which is imposed either by weight-clipping~\cite{arjovsky2017wasserstein} or gradient penalty methods~\citep{gulrajani2017improved}.
Our DDEs explicitly impose Gaussian-smoothness on the data distribution, which guarantees that the density is non-zero everywhere.
Additionally, the DDEs are trained to exactly constrain their gradients with respect to their inputs (Equation~\ref{eq:ddeloss}), without requiring additional techniques to control gradient magnitudes or weight clipping.

\section{Experiments}

\paragraph{2D Comparisons.}


Similar to~\citet{grathwohl2019ffjord}, we perform experiments for 2D density estimation and visualization over three datasets.
Additionally, we learn generative models. For our DDE networks, we used multi-layer perceptrons with residual connections. All networks have 25 layers, each with 32 channels and Softplus activation. For training we use 2048 samples per iteration to estimate the expected values.
Figure~\ref{fig:2Ddensity} shows the comparison of our method with Glow~\citep{Kingma2018GLOW}, BNAF~\citep{de2019block}, and FFJORD~\citep{grathwohl2019ffjord}.
Our DDEs can estimate the density accurately and capture the underlying complexities of each density.
Due to inherent smoothing as in kernel density estimation (KDE), our method induces a small blur to the density compared to BNAF. To demonstrate this effect, we show DDEs trained with both small and large noise standard deviations $\sigma_{\eta}=0.05$ and $\sigma_{\eta}=0.2$.
However, our DDE can estimate the density coherently through the data domain, whereas BNAF produces noisy approximation across the data.

Generator training and sampling is demonstrated in Figure~\ref{fig:2Ddensity} on the right. The sharp edges of the checkerboard samples imply that, due to invertibility of a small Gaussian blur, the generator learns to sample from the sharp target density even though the DDEs estimate noisy densities. While the generator update in theory requires DDE networks to be optimal at each gradient step, we take a limited number of 10 DDE gradient descent steps for each generator update to accelerate convergence. We summarize the training parameters used in these experiments the supplementary material.

\begin{figure*}[t]
\centering
\small
\setlength\tabcolsep{1.5pt}
\begin{tabular} {cccccccc}
& Real samples & Glow & BNAF & FFJORD & Ours \scriptsize{($\sigma_\eta:0.05$)} & Ours \scriptsize{($\sigma_\eta:0.2$)} & Ours generated \\
\raisebox{3.0\normalbaselineskip}[0pt][0pt]{\rotatebox[origin=c]{90}{Two Spirals}} &
\includegraphics[width=.13\textwidth]{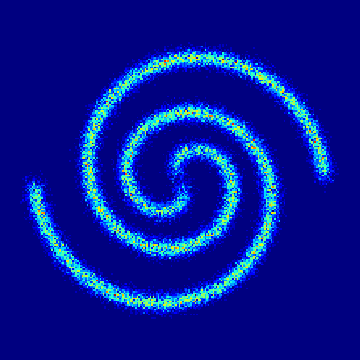} &
\includegraphics[width=.13\textwidth]{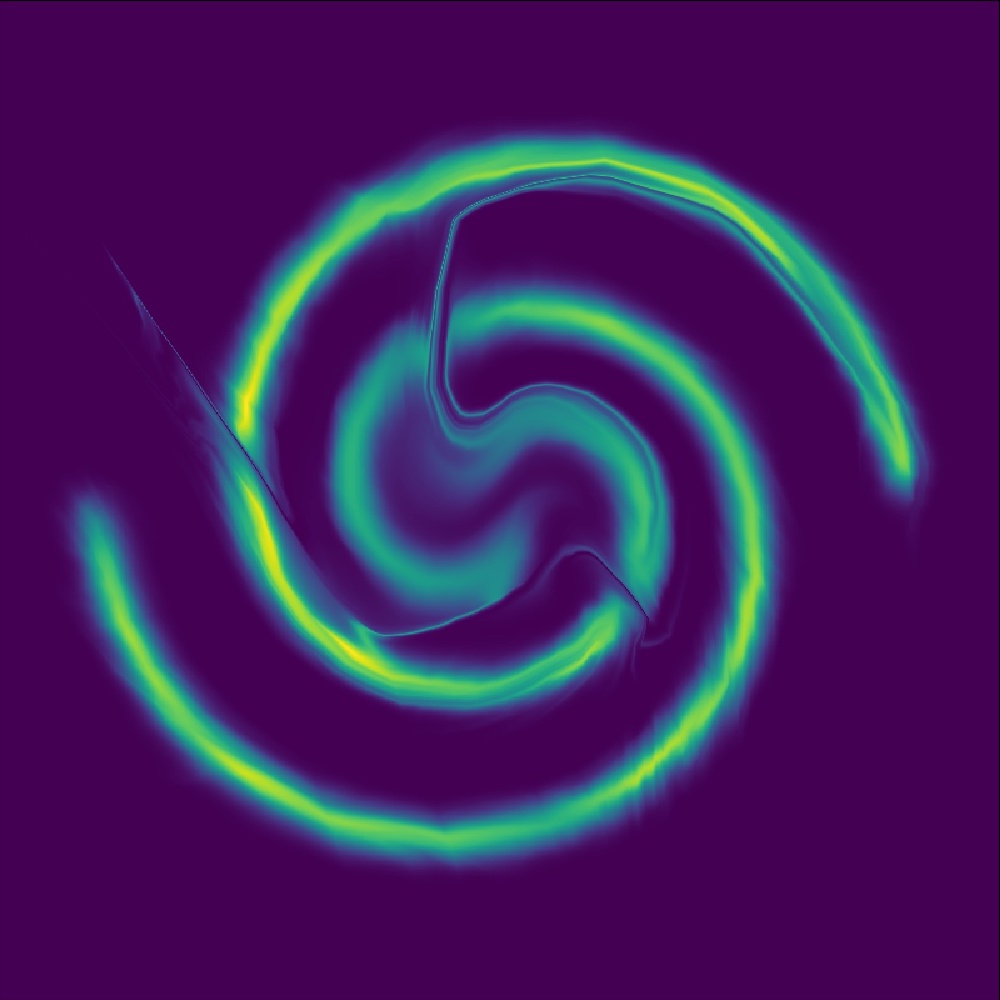} &
\includegraphics[width=.13\textwidth]{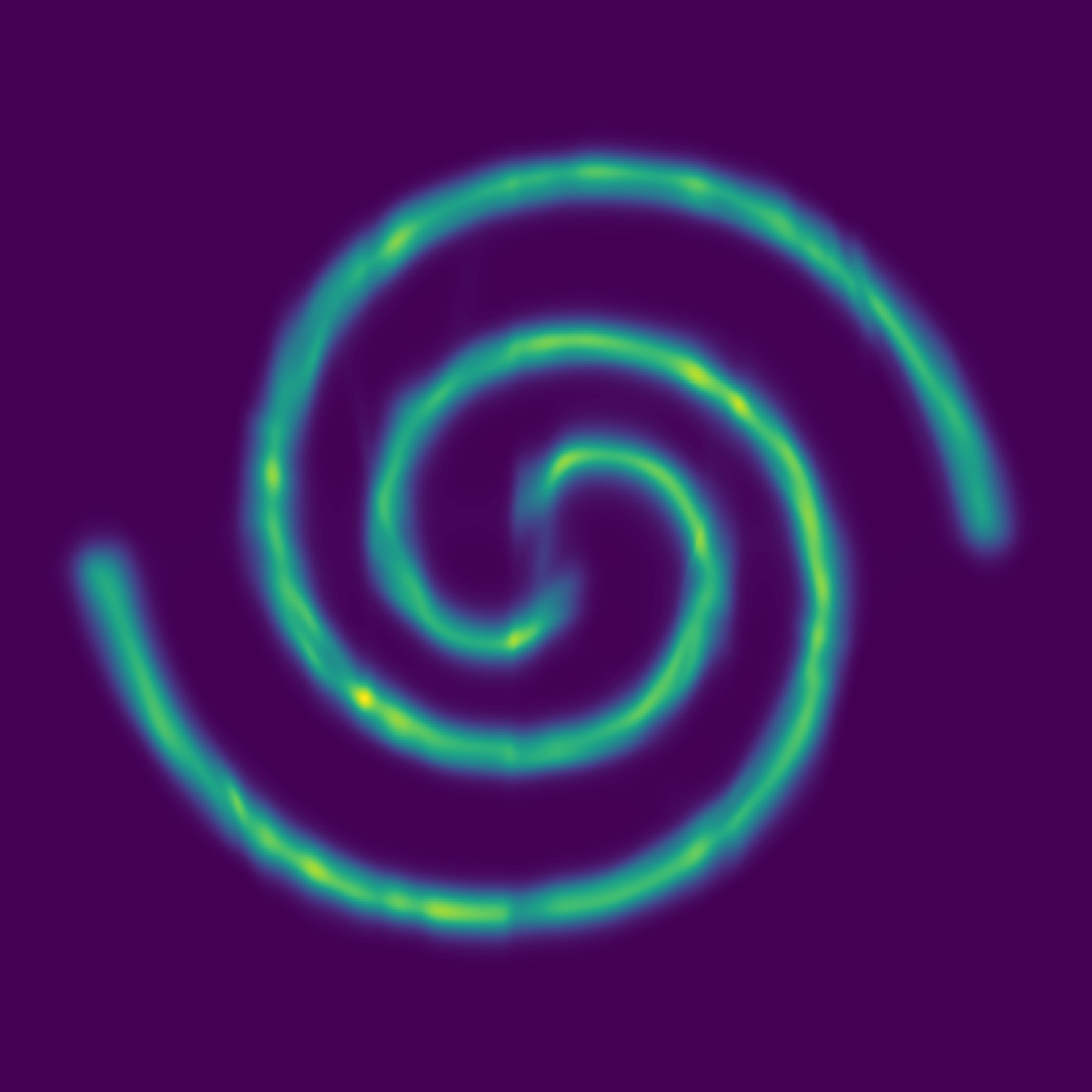} &
\includegraphics[width=.13\textwidth]{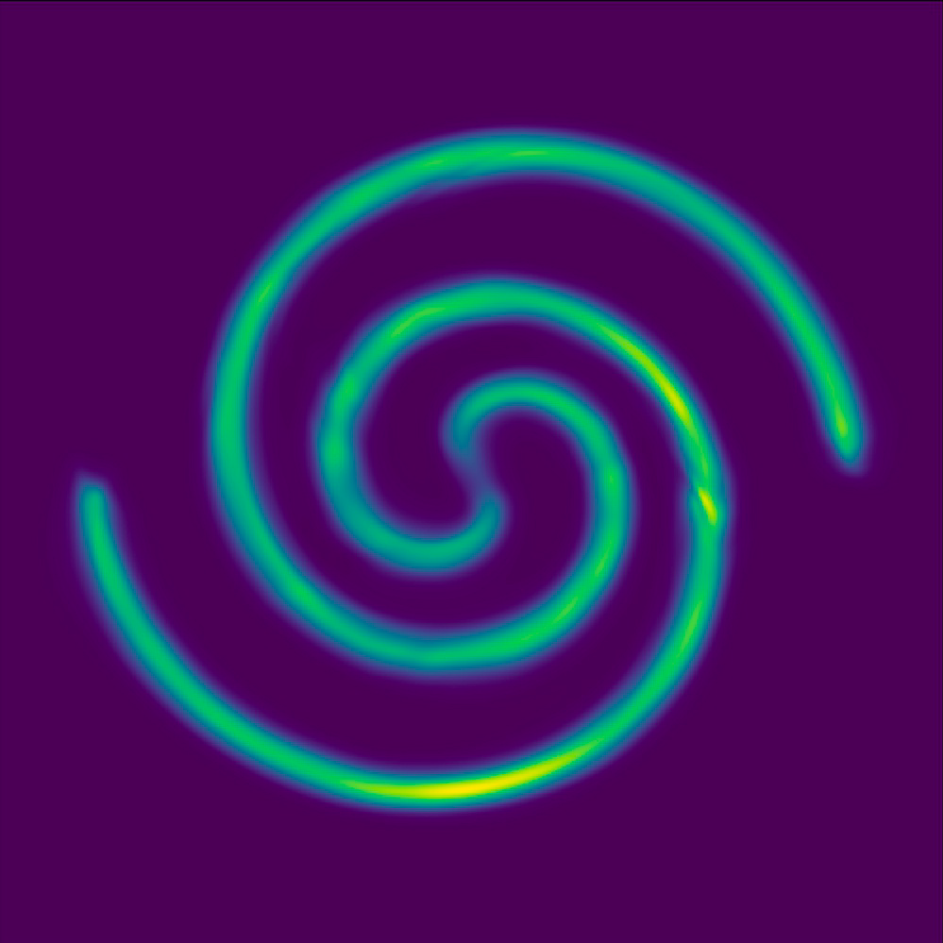} &
\includegraphics[width=.13\textwidth]{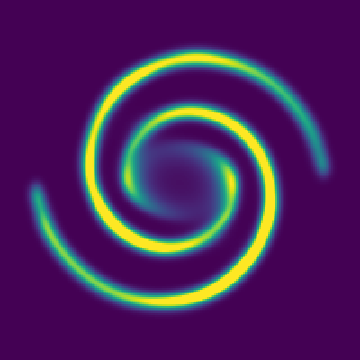} &
\includegraphics[width=.13\textwidth]{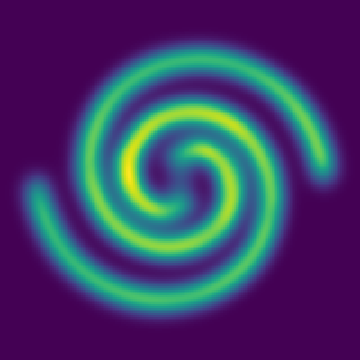} &
\includegraphics[width=.13\textwidth]{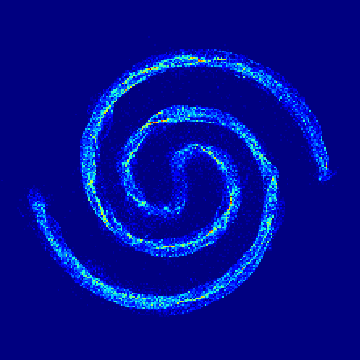} \\
\raisebox{3.0\normalbaselineskip}[0pt][0pt]{\rotatebox[origin=c]{90}{Checkerboard}} &
\includegraphics[width=.13\textwidth]{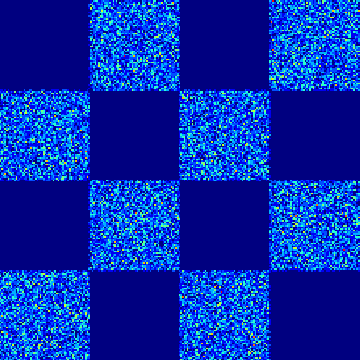} &
\includegraphics[width=.13\textwidth]{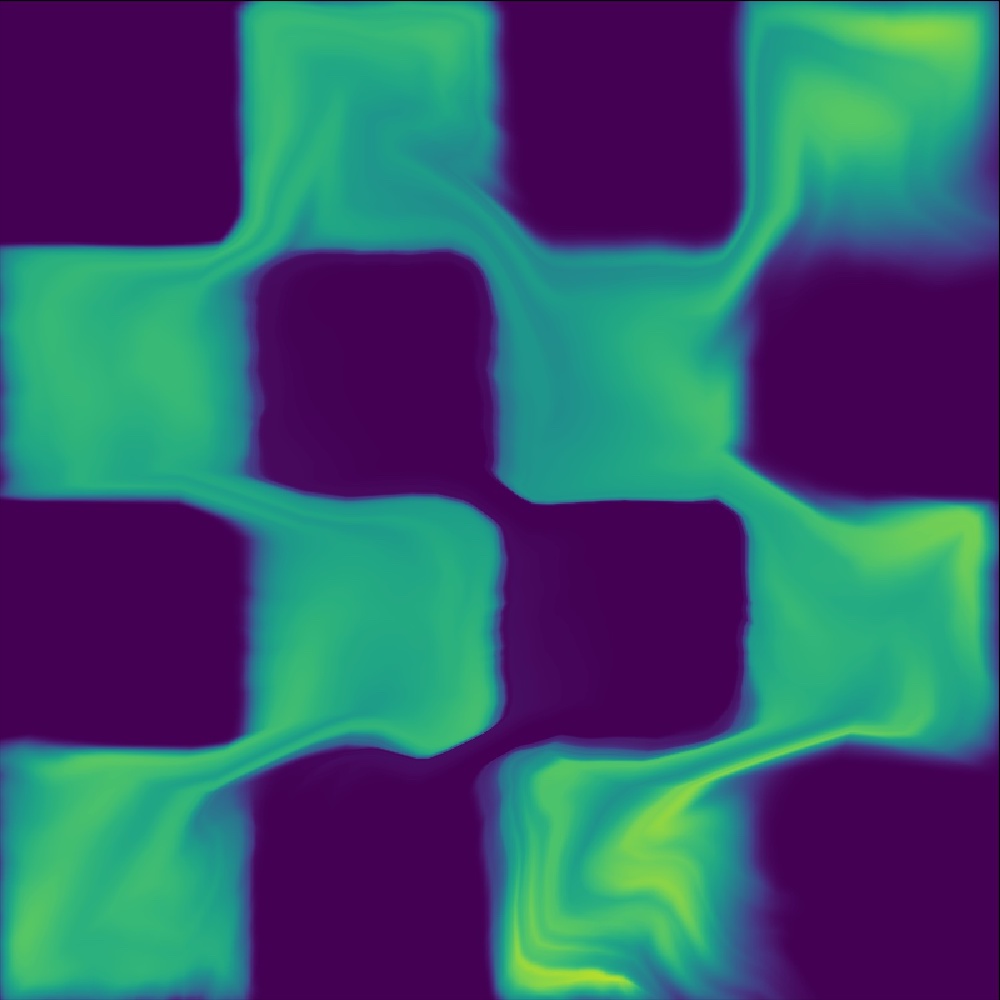} &
\includegraphics[width=.13\textwidth]{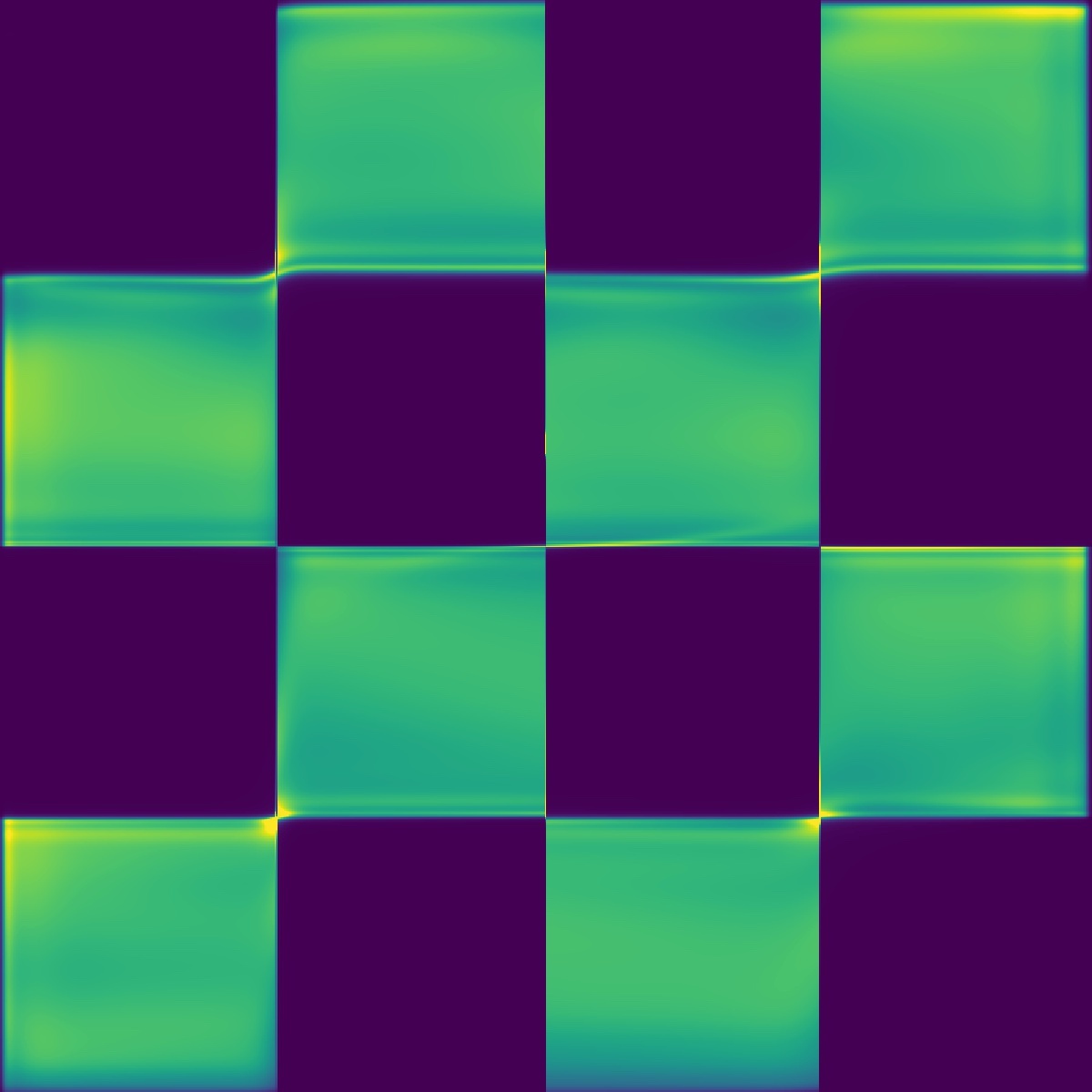} &
\includegraphics[width=.13\textwidth]{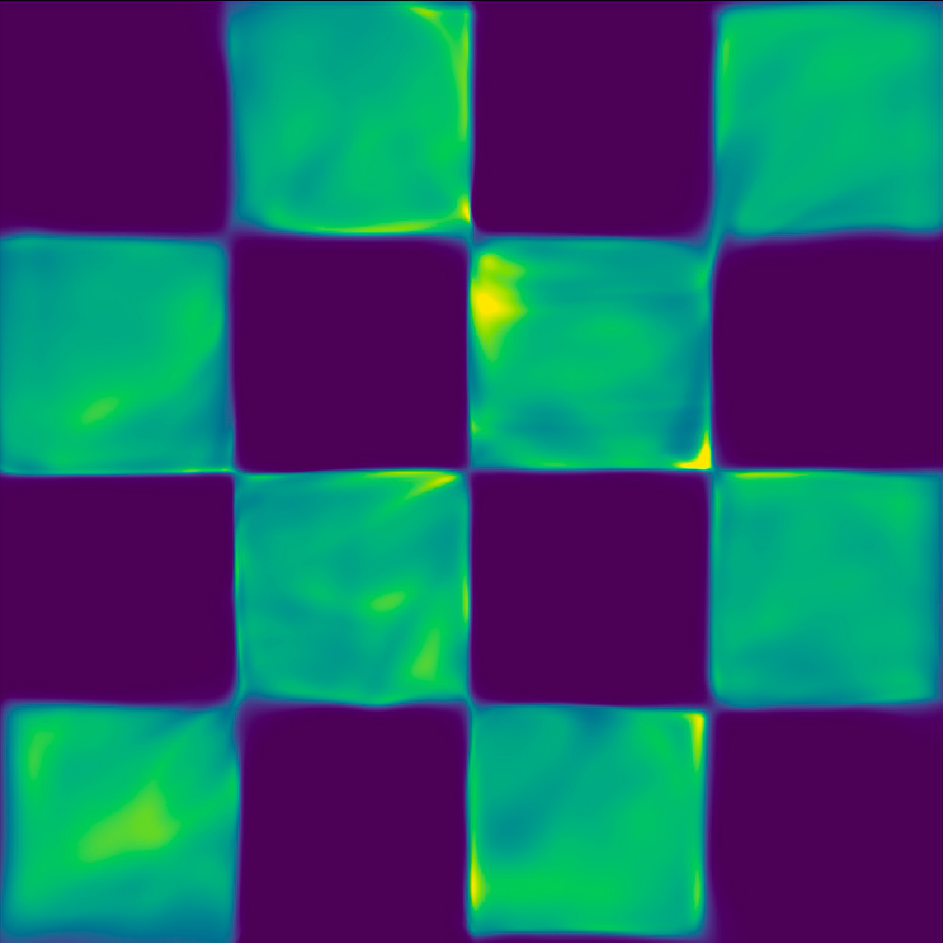} &
\includegraphics[width=.13\textwidth]{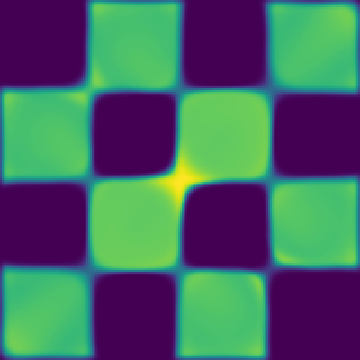} &
\includegraphics[width=.13\textwidth]{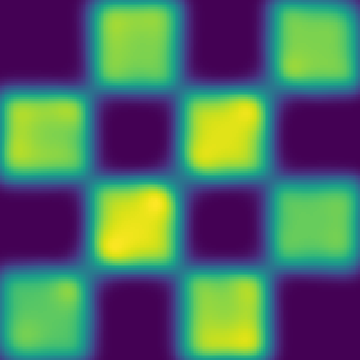} &
\includegraphics[width=.13\textwidth]{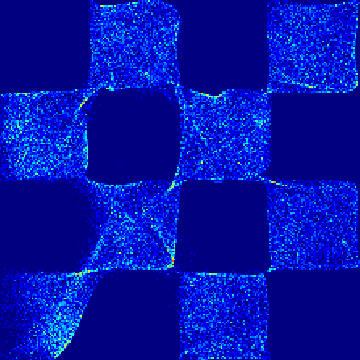}
\end{tabular}
\caption{Density estimation in 2D, showing that we can accurately capture these densities with few visual artifacts.
The rightmost column shows samples generated using our generative model training.}
\label{fig:2Ddensity}
\end{figure*}

\paragraph{MNIST.} Figure~\ref{fig:mnist} illustrates our generative training on MNIST~\citep{lecun1998mnist} using Algorithm~\ref{algo:generator}. We use a dense block architecture with fully connected layers here and refer to the supplementary material for the network and training details, including additional results for Fashion-MNIST~\citep{xiao2017fashion}. Figure~\ref{fig:mnist} shows qualitatively that our generator is able to replicate the underlying distributions. In addition, latent-space interpolation demonstrates that the network learns an intuitive and interpretable mapping from normally distributed latent variables to samples of the data distribution. 

 \begin{figure*}[t]
\centering
\begin{tabular} {cc}
(a) Generated samples & (b) Real samples\\
\includegraphics[width=.47\textwidth]{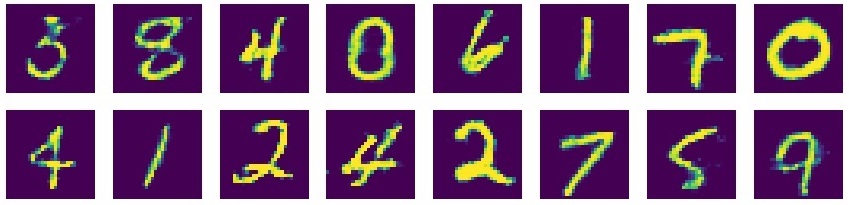} &
\includegraphics[width=.47\textwidth]{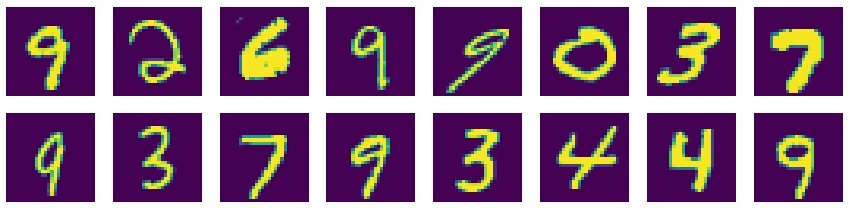} 
\end{tabular}  \\
(c) Interpolated samples using our model
\includegraphics[width=.98\textwidth]{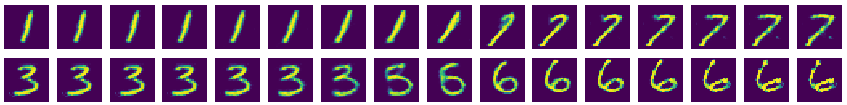} 
\caption[MNIST generated results.]
{
Generated MNIST (a), from the dataset (b), and latent space interpolation (c).
}
\label{fig:mnist}
\end{figure*}

\paragraph{CelebA.} Figure~\ref{fig:celebA} shows additional experiments on the CelebA dataset~\citep{liu2015deep}. 
The images in the dataset have $32 \times 32 \times 3$ dimensions and we normalize the pixel values to be in range $[-0.5, 0.5]$.
To show the flexibility of our algorithm with respect to neural network architectures, here we use a style-based generator~\citep{karras2019style} architecture for our generator network. Please refer to the supplementary material for network and training details. Figure~\ref{fig:celebA} shows that our approach can produce natural-looking images, and the model has learned to replicate the global distribution with a diverse set of images and different characteristics.

\begin{figure*}[t]
\centering
\begin{tabular} {cc}
(a) Generated samples & (b) Real samples\\
\includegraphics[width=.47\textwidth, trim={0 128px 0 0},clip]{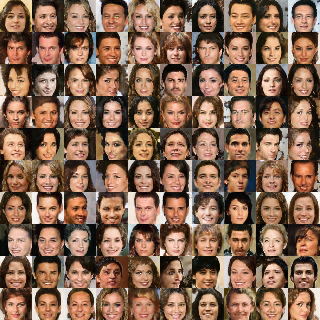} &
\includegraphics[width=.47\textwidth, trim={0 128px 0 0},clip]{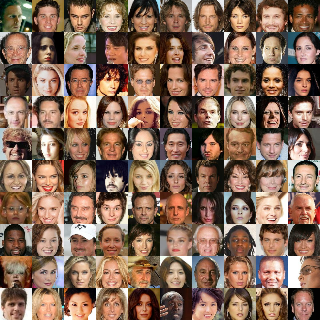} 
\end{tabular} 
\caption[celebA generated results.]
{
Generated results on $32 \times 32$ images from the celebA dataset~\citep{liu2015deep}.
}
\label{fig:celebA}
\end{figure*}

\paragraph{Quantitative Evaluation with Stacked-MNIST.} We perform a quantitative evaluation of our approach based on the synthetic Stacked-MNIST~\citep{metz2016unrolled} dataset, which was designed to analyse mode-collapse in generative models. The dataset is constructed by stacking three randomly chosen digit images from MNIST to generate samples of size $28\times28\times3$.
This augments the number of classes to $10^3$, which are considered as distinct modes of the dataset. Mode-collapse can be quantified by counting the number of nodes generated by a model. Additionally, the quality of the distribution can be measured by computing the KL-divergence between the generated class distribution and the original dataset, which has a uniform distribution in terms of class labels. Similar to prior work~\citep{metz2016unrolled}, we use an external classifier to measure the number of classes that each generator produces by separately inferring the class of each channel of the images. 

Figure~\ref{fig:stackedmnist} reports the quantitative results for this experiment by comparing our method with well-tuned GAN models. DCGAN~\citep{radford2015unsupervised} implements a basic GAN training strategy using a stable architecture. WGAN uses the Wasserstein distance~\citep{arjovsky2017wasserstein}, and WGAN+GP includes a gradient penalty to regularize the discriminator~\citep{gulrajani2017improved}. For a fair comparison, all methods use the DCGAN network architecture. Since our method requires two DDE networks, we have used fewer parameters in the DDEs so that in total we preserve the same number of parameters and capacity as the other methods.
For each method, we generate batches of 512 samples per training iteration and count the number of classes within each batch (that is, the maximum number of different labels in each batch is 512). We also plot the reverse KL-divergence to the uniform ground truth class distribution. Using the two measurements we can see how well each method replicates the distribution in terms of diversity and balance. Without fine-tuning and changing the capacity of our network models, our approach is comparable to modern GANs such as WGAN and WGAN+GP, which outperform DCGAN by a large margin in this experiment. 

We also report results for sampling techniques based on Score-Matching. We trained a Noise Conditional Score Network (NCSN) parametrized with a UNET architecture~\citep{ronneberger2015u}, which is then followed by a sampling algorithm using the Annealed Langevin Dynamics (ALD) as described by~\citet{Song2019GMG}. We refer to this method as UNET+ALD. We also implemented a model based on our approach called DDE+ALD, where we used our DDE network in combination with iterative Langevin sampling. While our training loss is equivalent to the score-matching objective, the DDE network outputs a scalar and explicitly enforces the score to be a conservative vector field by computing it as the gradient of its scalar output. DDE+ALD uses the spatial gradient of the DDE for iterative sampling with ALD~\citep{Song2019GMG}, instead of our proposed direct, one-step generator as described in Section~\ref{sec:generativemodels}. We observe that DDE+ALD is more stable compared to the UNET+ALD baseline, even though the UNET achieves a lower loss during training. We believe that this is because DDEs guarantee conservativeness of the distribution gradients (i.e. scores), which leads to more diverse and stable data generation as we see in Figure~\ref{fig:stackedmnist}. Furthermore, our approach with direct sampling outperforms both UNET+ALD and DDE+ALD.

\begin{figure*}[t]
\centering
\begin{tabular} {cc}
(a) Generated modes per batch & (b) KL-divergence\\
\includegraphics[width=.4\textwidth]{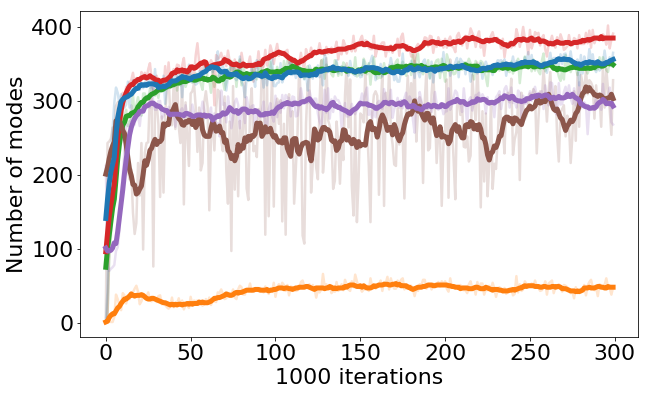} &
\includegraphics[width=.4\textwidth]{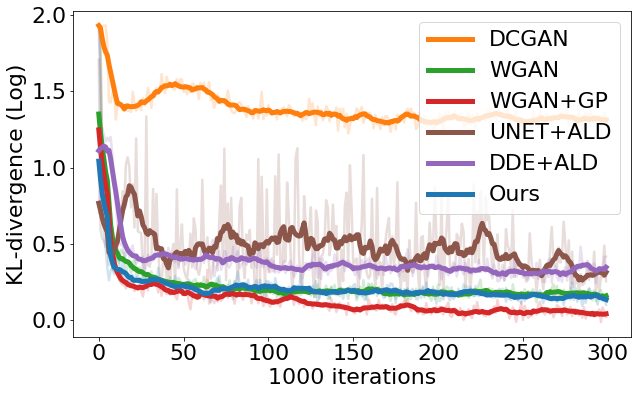} 
\end{tabular} 
\caption[Stacked-MNIST dataset results.]
{
Mode-collapse experiment results on Stacked-MNIST as a function of training iterations (for discriminator or DDE).
(a) Number of generated modes per batch of size 512.
(b) Reverse KL-divergence between the generated and the data distribution in the logarithmic domain.
}
\label{fig:stackedmnist}
\end{figure*}

\paragraph{Real Data Density Estimation.}

We follow the experiments in BNAF~\citep{de2019block} for density estimation, which includes the POWER, GAS, HEPMASS, and MINIBOON datasets~\citep{asuncion2007uci}. Since DDEs estimate densities up to their normalizing constant, we approximate the constant using Monte Carlo estimation here.
Similarly, \citet{li2019learning} use sampling to estimate the normalizing constant.
We show average log-likelihoods over test sets and compare to state-of-the-art methods for normalized density estimation in Table~\ref{tbl:densityestimation}.
We have omitted the results of the BSDS300 dataset~\citep{martin2001database}, since we could not estimate the normalizing constant reliably (due to high dimensionality of the data).
To train our DDEs, we used Multi-Layer Perceptrons (MLP) with residual connections between each layer.
All networks have 25 layers, with 64 channels and Softplus activations, except for GAS and HEPMASS, which employ 128 channels.
We trained the models for 400 epochs using learning rate of \num{2.5e-4} with linear decay with scale of $2$ every 100 epochs.
Similarly, we started the training by using noise standard deviation $\sigma_{\eta}=0.1$ and decreased it linearly with the scale of $1.1$ up to a dataset specific value, which we set to \num{5e-2} for POWER, \num{4e-2} for GAS, \num{2e-2} for HEPMASS, and $0.15$ for MINIBOON.
We estimate the normalizing constant via importance sampling using a Gaussian distribution with the mean and variance of the DDE input distribution.
We average 5 estimations using 51200 samples each (we used 10 times more samples for GAS), and we indicate the variance of this average in Table~\ref{tbl:densityestimation}. 

\newcommand{\addvar}[1]{{\tiny{$\pm #1$}}}
\newcommand{\descrcell}[2]{%
  \scriptsize
  \begin{tabular}[t]{@{}c@{}}\normalsize#1\\ \tiny{#2}\end{tabular}%
}


\begin{table*}[t]
\bgroup
\begin{center}
\begin{tabular}[c]{l c c c c }
\hlineB{3}
\raisebox{-0.35\normalbaselineskip}[0pt][0pt]{Model} & \descrcell{POWER}{$d=6, N\approx 2M$} & \descrcell{GAS}{ $d=8, N\approx 1M$} & \descrcell{HEPMASS}{$d=21, N\approx 500K$} & \descrcell{MINIBOON}{$d=43, N\approx 36K$ } \\
\hline
\cite{Dinh2017NVP} & $0.17$ \addvar{.01} & $8.33$ \addvar{.14} & $-18.71$ \addvar{.02} & $-13.55$ \addvar{.49} \\
\cite{Kingma2018GLOW} & $0.17$ \addvar{.01} & $8.15$ \addvar{.40} & $-18.92$ \addvar{.08} & $-11.35$ \addvar{.07} \\
\cite{germain2015made}* & $0.40$ \addvar{.01} & $8.47$ \addvar{.02} & $-15.15$ \addvar{.02} & $-12.27$ \addvar{.47} \\
\cite{papamakarios2017masked} & $0.24$ \addvar{.01} & $10.08$ \addvar{.02} & $-17.73$ \addvar{.02} & $-12.24$ \addvar{.45} \\
\cite{papamakarios2017masked}* & $0.30$ \addvar{.01} & $9.59$ \addvar{.02} & $-17.39$ \addvar{.02} & $-11.68$ \addvar{.44} \\
\cite{grathwohl2019ffjord} & $0.46$ \addvar{.01} & $8.59$ \addvar{.12} & $-14.92$ \addvar{.08} & $-10.43$ \addvar{.04} \\
\cite{Huang2018NAF} & $0.62$ \addvar{.01} & $11.96$ \addvar{.33} & $-15.09$ \addvar{.40} & $-8.86$ \addvar{.15} \\
\cite{oliva2018transformation} & $0.60$ \addvar{.01} & $\textbf{12.06}$ \addvar{.02} & $-13.78$ \addvar{.02} & $-11.01$ \addvar{.48} \\
\cite{de2019block} & $0.61$ \addvar{.01} & $\textbf{12.06}$ \addvar{.09} & $-14.71$ \addvar{.38} & $-8.95$ \addvar{.07} \\
\hline
\citet{li2019learning} & - & - & $\leq-20$  & $\leq-40$ \\
\textbf{Ours} & $\textbf{0.97}$ \addvar{.18} & $9.73$ \addvar{1.14} & $\textbf{-11.3}$ \addvar{.16} & $\textbf{-6.94}$ \addvar{1.81} \\
\hlineB{3}
\end{tabular}
\end{center}
\egroup
\caption{Average log-likelihood comparison in four datasets~\citep{asuncion2007uci}.
The top rows includes dataset size and dimensionality, bottom rows are normalized by sampling.
The upper section includes methods that estimate normalized densities.
Results of~\citet{li2019learning} are read from the bar plots reported in their article.
*Mixture of Gaussions. Best performances are in bold.}
\label{tbl:densityestimation}
\end{table*}

\paragraph{Discussion.}

Our approach relies on a key hyperparameter $\sigma_{\eta}$ that determines the training noise for the DDE, which we currently set manually. In the future we will investigate strategies to determine this parameter in a data-dependent manner. Another challenge is to obtain high-quality results using complex, high-dimensional data such CIFAR or high-resolution images. In practice, one strategy is to combine our approach with latent embedding learning methods~\citep{pmlr-v80-bojanowski18a}, in a similar fashion as proposed by~\citet{Hoshen_2019_CVPR}. The robustness of our technique with very high-dimensional data could potentially also be improved by leveraging slicing techniques~\citep{Song2019SSM,Wu_2019_CVPR}. Finally, we uses three networks to learn a generator (a DDE each for the input and generated data, and the generator). Our generator training approach, however, is independent of the type of density estimator, and techniques other than DDEs could also be used. 

\section{Conclusions}

We presented a novel approach to learn generative models using denoising density estimators (DDE), and our theoretical analysis proves that our training algorithm converges to a unique optimum.
Further, our technique 
does not require specific neural network architectures or ODE integration. We achieve state of the art results on a standard log-likelihood evaluation benchmark compared to recent techniques based on normalizing flows, continuous flows, and autoregressive models, and we demonstrate successful generators on diverse image data sets. Finally, a quantitative evaluation using the stacked MNIST data set shows that our approach avoids mode collapse similarly as state of the art Wasserstein GANs.


\section*{Broader Impact}

We propose a novel generative model that also provides a density estimate, which has potentially very broad applicability with many types of data. Generative models could be used to easily author visual media, for example, which would allow individuals or small teams to create content for education, training, or entertainment very easily. The density estimate could be used to leverage the generative model as a prior in highly underconstrained inverse problems, such as image or video restoration and various computational imaging techniques. Nefarious applications include deepfakes that attempt to mislead the consumers of the generated content. Generative models also replicate any biases that are inherent in the training data.

\bibliographystyle{apalike}
{\small
\bibliography{DDE}}

\end{document}